\documentclass[conf]{new-aiaa}
\usepackage[utf8]{inputenc}
\usepackage{url}

\usepackage{graphicx}
\usepackage{svg}
\graphicspath{{figures/}}
\usepackage{epsfig}
\usepackage{amsmath}

\usepackage{amsfonts}
\usepackage{amsthm}
\usepackage{comment}

\usepackage{bbm}
\usepackage{subcaption}
\usepackage{tikz}
\usepackage{tkz-euclide}

\usetikzlibrary{arrows,decorations.markings, calc, patterns, angles, quotes, intersections,decorations.shapes}
\usepackage{float}

\usepackage[ruled,vlined]{algorithm2e}

\newtheorem{theorem}{Theorem}[section]

\usepackage{mathtools}
\DeclarePairedDelimiter{\ceil}{\lceil}{\rceil}
\usepackage{aircraftshapes}

\usepackage[version=4]{mhchem}
\usepackage{siunitx}
\usepackage{longtable,tabularx}
\usepackage[normalem]{ulem}
\title{Decentralized Multi-target Tracking with Multiple Quadrotors using a PHD Filter}

\author{Aniket Shirsat\footnote{Ph.D. student, Mechanical and Aerospace Engineering, Arizona State University, Tempe, AZ 85287.} and Spring Berman.\footnote{Associate Professor, Mechanical and Aerospace Engineering, Arizona State University, Tempe, AZ 85287.}}
\affil{Arizona State University, Tempe, AZ 85287}
\begin{document}

\maketitle
\begin{abstract}
We consider a scenario in which a group of quadrotors is tasked at tracking multiple stationary targets in an unknown, bounded environment. The quadrotors search for targets along a spatial grid overlaid on the environment while performing a random walk on this grid modeled by a discrete-time discrete-state (DTDS) Markov chain. The quadrotors can transmit their estimates of the target locations to other quadrotors that occupy their current location on the grid; thus, their communication network is time-varying and not necessarily connected. We model the search procedure as a renewal-reward process on the underlying DTDS Markov chain. To accommodate changes in the set of targets observed by each quadrotor as it explores the environment, along with uncertainties in the quadrotors' measurements of the targets, we formulate the tracking problem in terms of Random Finite Sets (RFS). The quadrotors use RFS-based Probability Hypothesis Density (PHD) filters to estimate the number of targets and their locations. We present a theoretical estimation framework, based on the Gaussian Mixture formulation of the PHD filter, and preliminary simulation results toward extending existing approaches for RFS-based multi-target tracking to a decentralized multi-robot strategy for multi-target tracking. We validate this approach with simulations of multi-target tracking scenarios with different densities of robots and targets, and we evaluate the average time required for the robots in each scenario to reach agreement on a common set of targets.
\end{abstract}

\section{Introduction}
 Mobile ground robots \cite{buhmann1995mobile} and aerial robots \cite{achtelik2009autonomous} have often been used for exploration and mapping tasks. Heterogeneous teams of ground and aerial robots have been employed for applications that involve mapping an environment, such as disaster response \cite{michael2014collaborative,burgard2005coordinated} and surveillance \cite{grocholsky2006cooperative}. Such tasks require the robots to track features of interest that are present in the environment. Mobile robots, especially quadrotors, are subject to limitations on their operation due to constraints on the payloads that they can carry, including power, sensing and communication devices for transmitting information to other robots and/or to a command center. Many multi-robot control strategies rely on a centralized communication network for coordination. For example, some multi-robot exploration strategies, e.g. \cite{simmons2000coordination}, rely on constant two-way communication between the robots and a central node. Since a centralized communication architecture  is required, these strategies  do not scale well with the robot population size, as the communication bandwidth becomes a bottleneck with increasing numbers of robots. Moreover, a failure of the central node causes loss of communication for all the robots. Decentralized multi-robot control strategies can be used to  overcome these limitations. Such strategies involve only local communication between robots and scale well with the number of robots. However, communication among robots can become unreliable as the number of robots increases \cite{howard2006experiments}, and the communication network connectivity may be disrupted by the environment \cite{husain2013mapping} or by the movement of robots beyond communication range. 

Multi-target tracking is an established field of research with origins in the study of point processes \cite{daley2007introduction}, with most early applications in radar and sonar based tracking. In real-world scenarios, there is often uncertainty in the existence, locations, and dynamics of targets, as well as uncertainty in sensor measurements of targets that arise from sensor noise and false detections (clutter) around the real targets. Random Finite Set (RFS) models provide a probabilistic framework for multi-target tracking that can account for these uncertainties and ensure statistical guarantees on the accuracy of the estimated number of targets and their states.
Unlike RFS-based estimators, many classical probabilistic multi-target tracking approaches require techniques for data association, which is computationally intensive. Such approaches include multiple hypothesis tracking \cite{jensfelt2001active, reid1979algorithm}, in which an exhaustive search on all possible combinations of tracks and data associations is performed, and joint probabilistic data association \cite{bar1990tracking,schulz2001tracking}. The papers \cite{mahler2001random,mahler2003multitarget,mahler2007statistical} are foundational works on estimation methods based on Random Finite Sets, and they have made concepts from point process theory for multi-target tracking more accessible to engineering disciplines. The paper \cite{clark2006gm} presents an implementation of a multi-target  tracking approach as a probability hypothesis density (PHD) filter, and \cite{vo2006gaussian} provides examples of scenarios with targets that follow either linear or nonlinear motion models. In \cite{sung2017algorithm}, a PHD filter-based approach is proposed for simultaneous search  and tracking of an unknown number of moving targets by a single robot with a limited sensing field of view. There have also been works on multi-target search and tracking using multi-robot systems with communication networks that are always connected \cite{kamath2007triangulation,dames2013cooperative,dames2019distributed}  and that use decentralized controllers to maintain connectivity \cite{hung2016scalable}. In real-world applications, constraints on the robots' communication ranges limit the area that the robots 
can explore. Our previous work \cite{shirsat2020multi} presents a  consensus-based target search strategy for multiple robots that does not require the robots to maintain a connected communication network. The robots search for a static target while performing random walks on a discretization of the environment according to a discrete-time discrete-state (DTDS) Markov chain model. The strategy is proved to achieve consensus, defined as the robots' agreement on the presence of the target. However, the work \cite{shirsat2020multi} does not consider scenarios with multiple targets in the environment. 

This paper addresses the problem of tracking multiple targets without requiring a connected communication network using a multi-robot search strategy with the same probabilistic motion model and inter-robot communication constraints as in \cite{shirsat2020multi}.
We consider scenarios in which the quadrotors move according to a DTDS Markov chain model on a finite spatial grid, as illustrated in \autoref{fig:MotionStrategy_Visual}, while searching for multiple stationary targets. A robot detects the presence of a target by obtaining sensor measurements of target's states. We assume that the robots' sensors have limited fields of view and that the robots share information about the targets only with other robots within their local communication range.
We model the multi-robot target tracking procedure as a {\it renewal-reward process}, in which the reward is defined as the estimated number of targets and the targets' states, such as their spatial locations. 
The main contributions of this paper are the following:

\begin{enumerate}
\item 
Given a group of robots that explore an unknown bounded environment according to a DTDS Markov motion model, we prove that the number of a robot's encounters with any other robot up to a specific time characterizes a renewal process.
\item Using the Gaussian Mixture approximation of the PHD filter proposed in \cite{vo2006gaussian}, we show that under the constraint of local inter-robot communication, all robots eventually track the number of targets in the environment and their estimated states. 
\end{enumerate}

\begin{figure}[t]
    \centering
        \begin{tikzpicture}[scale=0.8]
            \draw [step=1cm, thin, gray!80] (0,0) grid (10,10);
            \filldraw[color=red!60!green] (3-0.2,5-0.2) rectangle (3+0.2,5+0.2);
            \filldraw[color=blue!50!green] (7-0.2,5-0.2) rectangle (7+0.2,5+0.2);
            \filldraw[color=red!40!green] (5-0.2,7-0.2) rectangle (5+0.2,7+0.2);
            \filldraw[color=blue!30!green] (5-0.2,3-0.2) rectangle (5+0.2,3+0.2);
            \filldraw[color=red!20!green] (3-0.2,4-0.2) rectangle (3+0.2,4+0.2);
            \filldraw[color=blue!10!green] (7-0.2,4-0.2) rectangle (7+0.2,4+0.2);
            \filldraw[color=black]  (5,5) circle (0.1);
            \node[right] (i) at (4.0,4.5) {$m=Y^1_k=Y^2_k$};
	        \node [quadcopter top,fill=white,draw=cyan,minimum width=0.5cm, rotate=45] at (2.6,1.0) {};
	        \filldraw[color=cyan] (2,1) circle (0.1);
	        \node[above] at (1.5,0.0) {$Y^1_0=i$}; 
            \node [quadcopter top,fill=white,draw=orange,minimum width=0.5cm, rotate=45] at (5,8.6) {};
            \filldraw[color=orange] (5,8) circle (0.1);
        	\node[above] at (4.1,7.3) {$Y^2_0 = j$};
        	\draw[->,color=cyan, thick,dashed] (2,1) -- (1,1);
        	\draw[->,color=cyan, thick,dashed] (1,1) -- (1,2);
        	\draw[->,color=cyan, thick,dashed] (1,2) -- (2,2); 
        	\draw[->,color=cyan, thick,dashed] (2,2) -- (2,3);
        	\draw[->,color=cyan, thick,dashed] (2,3) -- (2,4);
        	\draw[->,color=cyan, thick,dashed] (2,4) -- (3,4);
        	\draw[->,color=cyan, thick,dashed] (3,4) -- (4,4);
        	\draw[->,color=cyan, thick,dashed] (4,4) -- (4,5);
        	\draw[->,color=cyan, thick,dashed] (4,5) -- (5,5);
        	\draw[->,color=orange, thick,dashed] (5,8) -- (6,8);
        	\draw[->,color=orange, thick,dashed] (6,8) -- (7,8);
        	\draw[->,color=orange, thick,dashed] (7,8) -- (8,8);
        	\draw[->,color=orange, thick,dashed] (8,8) -- (8,7);
        	\draw[->,color=orange, thick,dashed] (8,7) -- (7,7);
        	\draw[->,color=orange, thick,dashed] (7,7) -- (7,6);
        	\draw[->,color=orange, thick,dashed] (7,6) -- (6,6);
        	\draw[->,color=orange, thick,dashed] (6,6) -- (6,5);
        	\draw[->,color=orange, thick,dashed] (6,5) -- (5,5); 
        	\node [quadcopter top,fill=white,draw=black, minimum width=1.0cm, rotate=45] at (8,2) {};
            \draw [->, color=black, thick] (8,3) -- node[right] {up} (8,4);
            \draw [->, color=black, thick] (7,2) -- node[below]{left} (6,2);
            \draw [->, color=black, thick] (8,1) -- node[right]{down} (8,0);
            \draw [->, color=black, thick] (9,2) -- node[below]{right} (10,2);
        \end{tikzpicture}
        \caption{Illustration of our multi-robot multi-target tracking strategy, showing sample paths for two quadrotors (orange and blue) on a square grid. The quadrotors search the environment for a set of static targets, represented by the squares, as they perform a random walk on the grid.} 
        \label{fig:MotionStrategy_Visual}
        \vspace*{-3mm}
\end{figure}

The remainder of the paper is organized as follows. Section \ref{sec:ProbForm} presents the problem statement. Section \ref{sec:DTDS_Markov_Chain} discusses the DTDS Markov chain motion model that the robots follow.
We then provide a brief discussion on the theory of renewal-reward processes and their application to our problem in Section \ref{sec:Renew-Reward}. In Section \ref{sec:RFSPHD}, we present the Random Finite Set formulation and its first-order moment, the Probability Hypothesis Density (PHD) filter, as an estimation framework for detecting and tracking multiple targets. In Section \ref{subsec:GMPHD Filter}, we describe the Gaussian Mixture approximation of the PHD filter from \cite{vo2006gaussian}. We then validate our strategy with simulations in Section \ref{sec:SimResults} and finally state our conclusions and future work in Section \ref{sec:Conc}.

\section{Problem Formulation}\label{sec:ProbForm}
We consider an unknown, bounded environment that contains a finite, non-zero number of static targets, indexed by the set $\mathcal{I} \subset \mathbb{Z}_{+}$. The environment is discretized into a square grid, and the four vertices of each grid cell are referred to as {\it nodes}. Let $\mathcal{S} \subset \mathbb{Z}_{+}$ denote the set of $S$ nodes, and let $\mathcal{G}_s = (\mathcal{V}_s, \mathcal{E}_s)$ be an undirected graph associated with this grid, where $\mathcal{V}_s$ is the set of nodes and $\mathcal{E}_s$ is the set of edges $(i,j)$ that signify the pairs of nodes $i,j \in \mathcal{V}_s$ between which the quadrotors can travel. A group of $N$ quadrotors, indexed by the set $\mathcal{N}$, explores the environment using a random walk strategy: each quadrotor performs a random walk on the grid, moving from its current node $i$ to an adjacent node $j$ with transition probability $p_{ij}$ at each discrete time $k$. We assume that each quadrotor is able to localize itself in this environment; i.e., that it knows which node it currently occupies. We also assume that quadrotors can communicate with one another only if they occupy the same node. We also assume that the quadrotors have perfect localization.

The number of targets estimated by each quadrotor is updated at every time step $k$. Let the $i^{th}$ target detected by quadrotor $a_j$ at time $k$ be $m_{i,k}^{a_j} \in \mathbb{R}_{+}$ which is a tuple, composed of the {\it state} of the target, which is a  time-varying property of the target like its location within the quadrotor's field of view (FoV), the pixels that it occupies in the quadrotor's camera image, and a unique identification label. Let $\mathcal{M}^{a_i}_k = \{ m_{1,k}^{a_i}, \ldots, m_{n_m,k}^{a_i}\}$ be the set of states of all targets detected by the quadrotor $a_j$ at time $k$, where $n_m$ is the maximum number of features that a quadrotor can detect simultaneously. The value of $n_m$ is limited by the computational capabilities and the available memory on the robot. As the quadrotor explores the environment, the number of targets that it detects and their states vary, as new targets appear in the FoV of the quadrotor and existing targets disappear. An {\it observation set} obtained by a quadrotor at a particular time consists of both measurements that are associated with actual targets and measurements arising from clutter. The objective of multi-target tracking is to jointly estimate, at each time step, the number of targets and the targets' states from a series of noisy and cluttered observation sets. The concept of a {\it random finite set (RFS)} is useful for formulating this problem, since within the FoV of a quadrotor, the number of targets and their states are time-varying and not completely known. A random finite set, as defined in \cite{mahler2007statistical}, is a set with a random number of elements which are themselves random. In other words, a RFS is a random variable whose possible values are unordered finite sets.
A computationally tractable approach to set-based estimation is to utilize the first statistical moment of an RFS, known as the Probability Hypothesis Density (PHD) or its \textit{intensity function}, for multi-target tracking. We propose to use the Gaussian Mixture formulation of the PHD filter (GM-PHD) for each quadrotor, as it is less computationally expensive than the particle filter implementation. 

\autoref{fig:MotionStrategy_Visual} illustrates our multi-target tracking approach with two quadrotors and six stationary targets. The quadrotors explore the grid according to the random walk motion model defined in Section \ref{sec:DTDS_Markov_Chain}, and they estimate the number of targets and their positions within their limited sensing FoV using the GM-PHD filter described in Section \ref{subsec:GMPHD Filter}. Sample trajectories are shown for each quadrotor as a sequence of arrows that indicate its direction of motion. At time step $k$, the first {\it renewal epoch}, the quadrotors meet at node $m$ and exchange {\it rewards}, defined as each quadrotor's estimates of the number of targets that it has detected up until time $k$ and their positions, as described in Section \ref{sec:Renew-Reward}. The implementation of this strategy is described in pseudocode in Algorithm \ref{algo:GMPHD_with_DTDS_Markov_chain} and Algorithm \ref{algo:RenewalReward}. We extract only unique target states during simulation by using set union methods, as described in Algorithm \ref{algo:UniqueSetTuplesConsensus}.

\begin{algorithm}[h]
\SetAlgoLined
\DontPrintSemicolon
\SetNoFillComment
\textbf{Step 0: Initialization}
$a_{i}$, $J_{\gamma}^{(a_{i})},$ $\mu_{\gamma}^{(a_{i})},$ $P_{\gamma}^{(a_{i})},$  $Y_0^{(a_{i})},$ $F_{k-1},$ $Q_{k-1},$ $H_k,$ $R_k,$ $\kappa^{(a_{i})}_{0}(z),$  $p_{\mathbf{S}},$  $p_{\mathbf{D}},$  $w^{(a_{i})}_{0},$ $\mathcal{M}^{(\cdot)}_{0}$\\
\textbf{Step 1: Random Walk}\\
~~[$Y^{(a_{i})}_{k}$] = MarkovRandomWalk($Y^{(a_{i})}_{k-1}$);\\
\textbf{Step 2: GM-PHD Filter}
\begin{enumerate}
    \item[{a}] {\underline{\textit{Predicted State Components}} Apply steps 1 and 2 from Table 1 in \cite{vo2006gaussian}}
    \item[] { $J^{(a_{i})}_{k-1}$=$J^{(a_{i})}_{\gamma}$;~ $w^{(a_{i})}_{k-1}$=$w^{(a_{i})}_{\gamma}$;~ $\mu^{(a_{i})}_{k-1}$ = $\mu^{(a_{i})}_{\gamma}$}
    \item[] {[$w^{(a_{i})}_{k|k-1}$, $\mu^{(a_{i})}_{k|k-1}$, $P^{(a_{i})}_{k|k-1}$,$J^{(a_{i})}_{k|k-1}$] = predictGMPHD($J_{k-1}$,$w^{(a_{i})}_{k-1}$,$\mu^{(a_{i})}_{k-1}$,$P^{(a_{i})}_{k-1}$,$F_{k-1}$,$Q_{k-1}$,$p_{\mathbf{S}}$);}
    \item[{b}] {\underline{\textit{Updated State Components}} Apply steps 3 and 4 from Table 1 in \cite{vo2006gaussian}}
    \item[] {[$w_k$,$\mu_k$,$P_k$,$J_k$] =updateGMPHD($H_k$,$\mu_{k|k-1}$,$R_k$,$P_{k|k-1}$,$p_{\mathbf{D}}$,$J_{k|k-1}$,$w_{k|k-1}$, $\kappa_{k}^{(a_i)}(z)$,$Z_k^{(a_i)}$);} 
    \item[{c}] {\underline{\textit{Pruning and Merging Components}} Apply all steps from Table 2 in \cite{vo2006gaussian}}
    \item[{d}] {\underline{\textit{Multi-target State Extraction}} Apply all steps from Table 3 in \cite{vo2006gaussian} with $thresh_{state} = 0.5$}
    \item[] {[$\hat{X}_{k}^{(a_i)}$,$\hat{w}_{k}^{(a_i)}$,$\hat{P}_{k}^{(a_i)}$] = extractMTStateGMPHD($w_k^{(a_i)}$,$\mu_k^{(a_i)}$,$P_k^{(a_i)}$,$thresh_{state}$);}
    \item[]{$\mathcal{M}^{(a_i)}_k$ = $\mathcal{M}^{(a_i)}_{k-1} \bigcup \hat{X}_{k}^{(a_i)}$};
\end{enumerate}
\caption{Control strategy for robot $a_{i} \in \mathcal{N}$}
\label{algo:GMPHD_with_DTDS_Markov_chain}
\end{algorithm}

\section{Discrete-Time Discrete-Space (DTDS) Markov Chain Model of Robot Motion }\label{sec:DTDS_Markov_Chain}
Let $Y^{a_i}_{k} \in \mathcal{S}$ be the random variable that represents the location of quadrotor $a_i$ at time $k$ on the spatial grid. For each quadrotor $a_i$, the probability mass function $\pi_k \in \mathbb{R}^{1 \times S}$ of $Y^{a_i}_k$ evolves according to a discrete-time discrete-space (DTDS) Markov chain given by:
  \begin{equation}
      \pi_{k+1} = \mathbf{P} \pi_k,
      \label{eqn:DTDSMarkovChain}
  \end{equation}
where the \textit{state transition matrix} $\mathbf{P} \in \mathbb{R}^{S \times S}$ has elements $p_{ij} \in [0,1]$ at row $i$ and column $j$. The time evolution of the  probability mass function of $Y^{a_i}_k$ is expressed using the Markov property as follows: 
\begin{equation}
    Pr(Y^{a_i}_{k+1}=j_{k+1} | Y^{a_i}_{k}=j_{k},\ldots, Y^{a_i}_{0}=j_{0}) ~~=~~ Pr (Y^{a_i}_{k+1} = j_{k+1} | Y^{a_i}_{k} = j_{k} ),
    \label{eqn:MarkovProperty}
\end{equation}
where $j_k$ is a specific node in the spatial grid that the quadrotor may occupy at time $k$.
In other words, \eqref{eqn:MarkovProperty} states that the future location of the quadrotor depends only on its current location and is statistically independent of any previous locations. We assume that the DTDS Markov chain is time-homogeneous, which implies that  $Pr(Y^{a_i}_{k+1} = j_{k+1} | Y^{a_i}_{k} = j_{k})$ is same for all quadrotors at all time steps. Thus, the entries of $\mathbf{P}$ can be defined as follows: 
\begin{equation}
    p_{ij} = Pr(Y^{a_i}_{k+1} = j_{k+1} | Y^{a_i}_{k} = j_{k}), \hspace{2mm}\forall j_{k} \in S,~ k \in \mathbb{Z}_{\geq 0},~ a_i \in \mathcal{N}. 
\end{equation}
Assuming that each quadrotor chooses its next position from a uniform random distribution, we can compute the entries of $\mathbf{P}$ as follows:
\begin{equation}
         p_{ij} =\begin{cases} 
         \frac{1}{d_{i}+1}, & (i,j) \in \mathcal{E}_{s}, \\
          0, & $otherwise$, 
    \end{cases}
    \label{eqn:TransitionMat_Elements}
\end{equation}
where $d_{i}$ is the degree of the node $i \in \mathcal{S}$. Since each entry $p_{ij} \geq 0$, we use the notation $\mathbf{P} \geq 0$. We see that $\mathbf{P}^{m} \geq 0$ for $m \geq 1$. Hence, $\mathbf{P}$ is a \textit{non-negative matrix}. 
Then, from Theorem 5 in \cite{grimmett2001probability}, we can say that $\mathbf{P}$ is a stochastic matrix.
We define \autoref{eqn:DTDSMarkovChain} as the {\it spatial Markov chain}. From the construction of the spatial Markov chain, every quadrotor has a positive probability of moving from node $i \in S$ to any node $j \in S$ in a finite number of time steps. Thus, the Markov chain is said to be  \textit{irreducible}, and consequently, $\mathbf{P}$ is an \textit{irreducible} matrix. Now applying Lemma 8.4.4 in \cite{horn1990matrix}, we know that there exists a real unique positive left eigenvector of $\mathbf{P}$. Since $\mathbf{P}$ is a stochastic matrix, we have that $\rho(\mathbf{P}) = 1$, where $\rho(\mathbf{P})$ denotes the spectral radius of $\mathbf{P}$. Thus, we can conclude that this real unique positive left eigenvector is the \textit{stationary distribution} associated with the spatial Markov chain. 
Since we have shown that the Markov chain is irreducible and has a stationary distribution $\pi$ that satisfies $\pi \mathbf{P} = \pi$, we can conclude from Theorem 21.12 in \cite{levin2017markov} that the Markov chain is {\it positive recurrent}. 
Thus, all states in the Markov chain are positive recurrent, which implies that each quadrotor will keep visiting every state on the finite spatial grid infinitely often. We will use this result to prove results on the associated renewal-reward process, which is discussed next.

\section{Renewal-Reward Process}\label{sec:Renew-Reward}
We now define a random variable  $\tau^{a_i}_{j} \in \mathbb{R}_{\geq 0}$ as the $j^{th}$ interval between two successive times at which quadrotor $a_i$ and any another quadrotor occupy the same node. This time interval is referred to as the {\it inter-arrival time}. A {\it renewal epoch} is a time instant at which two quadrotors meet at the same node. For each quadrotor $a_i$, we define the counting process $T^{a_i}(k) \in \mathbb{Z}_{\geq 0}$ as the number of times $a_i$ has met any other quadrotor by time $k$. At each renewal epoch, quadrotor $a_i$ updates its {\it reward},  defined as the number of all detected targets and their locations, with the number of targets and locations detected by the quadrotor(s) that occupies its current node and transmits this information to $a_i$.
We use the definition of a {\it renewal process} given by Definition 7.1 in \cite{ross2014introduction}. If the sequence of non-negative random variables $\{ \tau^{a_i}_{0}, \tau^{a_i}_{1}, \ldots, \}$ is independent and identically distributed, then the counting process $T^{a_i}(k)$ is said to be a renewal process. We demonstrate that $T^{a_i}(k)$ is a renewal process at the end of this section. 

For a renewal process having inter-arrival times  $\tau^{a_i}_{0}, \tau^{a_i}_{1}, \ldots $, we define $S_n^{a_i} = \sum_{j=1}^{n} \tau^{a_i}_{j}$ as the $n^{th}$ renewal epoch, with $S_{0}^{a_i} = 0$ for all $a_{i} \in \mathcal{N}$. From the definition of a renewal process, we can infer that the number of renewal epochs by time $k$ is greater than or equal to $n$ if and only if the $n^{th}$ renewal epoch occurs before or at time $k$; that is,
\begin{equation}
    T^{a_i}(k) \geq n  ~~\Leftrightarrow~~ S_n^{a_i} \leq k.
    \label{eqn:renewaltointerarrivalreln}
\end{equation}
Now consider that at each renewal epoch, quadrotor $a_i$ receives a reward. The reward $R^{a_i}_{n}$  earned by quadrotor $a_i$ when the $n^{th}$ renewal occurs is defined as follows: 
\begin{equation}
    R_n^{a_i} = \mathcal{M}^{a_i}_{k} \bigcup_{a_j \neq a_i}  \mathcal{M}^{a_j}_{k}, ~~Y^{a_i}_k = Y^{a_j}_k ~\text{and}~ a_j \in \mathcal{N}.\\
    \label{eqn:MeetingReward}
\end{equation}
\autoref{eqn:renewaltointerarrivalreln} and \autoref{eqn:MeetingReward} together define a \textit{renewal-reward process.} Each quadrotor $a_i$ calculates $\mathcal{M}^{a_i}_k$ by estimating the number of targets and their spatial distribution using a PHD filter. In Section \ref{sec:RFSPHD}, we describe some fundamental theory on target detection and tracking using this type of filter.

Given the quadrotor motion model defined in Section \ref{sec:DTDS_Markov_Chain}, we can model the dynamics of all the quadrotors' movements on the spatial grid by a composite Markov chain with states $\psi_k = (Y_k^{a_1}, Y_k^{a_2}, \ldots, Y_k^{a_N}) \in \mathcal{H}$, where $\mathcal{H} = \mathcal{S}^{\mathcal{N}}$. Note that $S = |\mathcal{S}|$ and $|\mathcal{H}| = S^N$. We now define another undirected graph $\hat{\mathcal{G}} = (\hat{\mathcal{V}},\hat{\mathcal{E}})$ associated with this composite Markov chain. The vertex set $\hat{\mathcal{V}}$ is a set of all possible realizations $\hat{\imath} \in \mathcal{H}$ of $\mathbf{\psi}_k$. Here $\hat{\imath}(a_{l})$ represents the $a_{l}^{th}$ entry of $\hat{\imath}$, which corresponds to the spatial node $i \in \mathcal{S}$ occupied by robot $a_{l} \in \mathcal{N}$ and $l \in \mathcal{I}$. We define the edge set $\hat{\mathcal{E}}$ of graph $\hat{\mathcal{G}}$ as follows:
$(\hat{\imath},\hat{\jmath}) \in \hat{\mathcal{E}}$ if and only if $(\hat{\imath}(a_{l}),\hat{\jmath}(a_{l})) \in \mathcal{E}_s$ for all robots $a_{l} \in \mathcal{N}$. Let $\mathbf{Q} \in \mathbb{R}^{ | \mathcal{H}| \times |\mathcal{H}|}$ be the state transition matrix associated with this composite Markov chain. An element of $\mathbf{Q}$, denoted by $q_{\hat{\imath} \hat{\jmath}}$, is the probability that in the next time step, each robot $a$ will move from spatial node $\hat{\imath}(a_{l})$ to node $\hat{\jmath}(a_{l})$. These elements are computed from the transition probabilities defined by Equation \eqref{eqn:TransitionMat_Elements} as follows: 
\begin{equation}
    q_{\hat{\imath} \hat{\jmath}} = \prod_{a_{l}=1}^{N} p_{\hat{\imath}(a_{l}) \hat{\jmath}(a_{l})}, ~~~~ \forall \hat{\imath}, \hat{\jmath} \in \mathcal{H}~~ \& ~~ l \in \mathcal{I}. 
    \label{eqn:TransitionRatesProductState}
\end{equation}

As an illustration, consider a set of two robots, $\mathcal{N} = \{a_1,a_2\}$, that move on the graph $\mathcal{G}_s$  shown in \autoref{fig:MarkovianTransition}.
The robots can stay at their current node in the next time step or travel between nodes $i$ and $j$ and between nodes $j$ and $l$, but they cannot travel between nodes $i$ and $l$.  
\autoref{fig:Q_Mat_elem_Vis} shows a subset of the resulting composite graph $\mathcal{\hat{G}}$, which has the set of nodes 
$\hat{\mathcal{V}} = \{(i,i), (i,j), (i,l), (j,i), (j,j), (j,l), (l,i), (l,j), (l,l)\}$. Each node in $\hat{\mathcal{V}}$ is labeled by a single index $\hat{\imath}$, e.g., $\hat{\imath} = (i,j)$, with $\hat{\imath}(a_1) = i$ and $\hat{\imath}(a_2) = j$. Given the connectivity of the spatial grid defined by $\mathcal{E}_s$, we can for example identify $((i,j),(i,l))$ as an edge in $\hat{\mathcal{E}}$, but not $((i,j),(l,l))$.
Since $N = 2$ and $S = 3$, we have that $|\mathcal{H}| = 3^2 = 9$. For each $\hat{\imath},\hat{\jmath} \in \hat{\mathcal{V}}$,
we can compute the transition probabilities in $\mathbf{Q} \in \mathbb{R}^{9 \times 9}$ from Equation \eqref{eqn:TransitionRatesProductState} as:  
\begin{equation}
    q_{\hat{\imath} \hat{\jmath}} = Pr\left(\mathbf{\mathbf{\psi}}_{k+1} = \hat{\jmath} ~|~ \mathbf{\psi}_{k} = \hat{\imath} \right) = p_{\hat{\imath}(a_1)\hat{\jmath}(a_1)}p_{\hat{\imath}(a_2)\hat{\jmath}(a_2)}, ~~ k \in \mathbb{Z}_+.
     \label{eqn:ProductSpaceTransitionProb}
\end{equation}

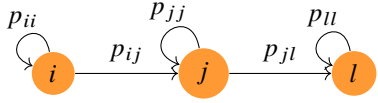
\begin{figure}[t]
    \centering
    \begin{tikzpicture}
        \node[style={circle,fill=orange!80}] (a1) at (0,0) {$i$};
        \node[style={circle,fill=orange!80}] (a2) at (2,0) {$j$};
        \node[style={circle,fill=orange!80}] (a3) at (4,0) {$l$};
        \node[] (a4) at (5,0){};
        \draw[->,style={draw=black}] (a1) -- node[above]{$p_{i j}$} (a2);
        \draw[->,style={draw=black}] (a2) -- node[above]{$p_{j l}$} (a3);
        \draw[->,style={draw=black}] (a1) to [in=160, out=100,looseness=5] node[above] {$p_{i i}$} (a1);
        \draw[->,style={draw=black}] (a2) to [in=160, out=100,looseness=5] node[above] {$p_{j j}$} (a2);
        \draw[->,style={draw=black}] (a3) to [in=160, out=100,looseness=5] node[above] {$p_{l l}$} (a3);
    \end{tikzpicture}
    \caption{An example graph $\mathcal{G}_s =  (\mathcal{V}_s,\mathcal{E}_s)$ defined on the set of spatial nodes  $\mathcal{V}_s = \{i,j,l\}$. The arrows signify directed edges between pairs of distinct nodes or self-edges. The edge set 
    is $\mathcal{E}_s = \{(i,i), (j,j), (l,l), (i,j), (j,l)\}$. }
    \label{fig:MarkovianTransition}
\end{figure} 
\begin{figure}[t]
    \centering
    \begin{tikzpicture}[scale=.5,auto=center]
        \node  [style={circle,fill=green!30!white}] (a1i1a2i1) at (0,0) {$(i,i)$};
        \node (ghatv1) at (0,1.7) {$\hat{i}$};
        \node [style={circle,fill=green!30!white}] (a1i1a2i2) at (5,0) {$(i,j)$};
        \node (ghatv2) at (5,1.7) {$\hat{j}$};
        \node  [style={circle,fill=green!30!white}] (a1i1a2i3) at (10,0) {$(i,l)$};
        \node (ghatv3) at (10,1.7) {$\hat{l}$}; 
        \node (ghatend) at (13,0){}; 
        \draw[->,style={draw=black}] (a1i1a2i1) -- node[above,midway]{$q_{\hat{i},\hat{j}}$} (a1i1a2i2);
        \draw[->,style={draw=black}] (a1i1a2i2) -- node[above,midway]{$q_{\hat{j},\hat{l}}$} (a1i1a2i3);
        \draw[style={draw=black, dashed}] (a1i1a2i3) -- (ghatend);
        \draw[->,style={draw=black}] (a1i1a2i1) to [loop, in=250, out=210,looseness=3] node[below] {$q_{\hat{i},\hat{i}}$} (a1i1a2i1);
        \draw[->,style={draw=black}] (a1i1a2i2) to [loop, in=250, out=210,looseness=3]
        node[below] {$q_{\hat{j},\hat{j}}$} (a1i1a2i2);
        \draw[->,style={draw=black}] (a1i1a2i3) to [loop, in=250, out=210,looseness=3] 
        node[below] {$q_{\hat{l},\hat{l}}$} (a1i1a2i3);
    \end{tikzpicture}
    \caption{A subset of the composite graph $\mathcal{\hat{G}}=(\mathcal{\hat{V}},\mathcal{\hat{E}})$ for two agents that 
    move on the graph $\mathcal{G}_s$ shown in \autoref{fig:MarkovianTransition}.}
    \label{fig:Q_Mat_elem_Vis}
\end{figure}
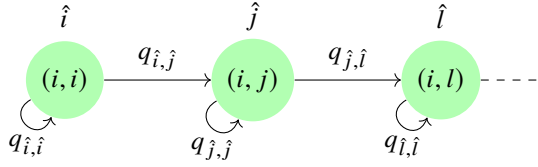

We now prove that $T^{a_i}(k)$ is a renewal process.
\begin{theorem}
$T^{a_i}(k)$ is a renewal process on the composite Markov chain $\mathbf{\psi}_k$.
\label{thm:Renewal_Thm_Composite_MC}
\end{theorem}
\begin{proof}
Suppose that an initial time instant $k_0$, the locations of all $N$ robots on the spatial grid are represented by the node $\hat{\imath} \in \hat{\mathcal{V}}$. Consider another set of robot locations at time $k_0 + k$, where $k>0$, represented by the node $\hat{\jmath} \in \hat{\mathcal{V}}$. The transition of the robots from configuration $\hat{\imath}$ to configuration $\hat{\jmath}$ in $k$ time steps corresponds to a random walk of length $k$ on the composite Markov chain $\mathbf{\psi}_k$ from node $\hat{\imath}$ to node $\hat{\jmath}$. It also corresponds to a random walk by each robot $a_i$ on the spatial grid from node $\hat{\imath}(a_i)$ to node $\hat{\jmath}(a_i)$ in $k$ time steps. By construction, the graph $\mathcal{G}_s$ is strongly connected and each of its nodes has a self-edge.
Therefore, there exists a discrete time $n>0$ such that, for each robot $a_i \in \mathcal{N}$, there exists a random walk on the spatial grid from node $\hat{\imath}(a_i)$ to node $\hat{\jmath}(a_i)$ in $n$ time steps. Consequently, there always exists a random walk of length $n$ on the composite Markov chain $\mathbf{\psi}_k$ from node $\hat{\imath}$ to node $\hat{\jmath}$. Therefore, $\mathbf{\psi}_k$ is an irreducible Markov chain. All states of an irreducible Markov chain belong to a single communication class. In this case, all states are \textit{positive recurrent}. As a result, $\mathbf{\psi}_k$ is \textit{positive recurrent}. Thus, each state in $\mathbf{\psi}_k$ is visited infinitely often from all other states in $\mathbf{\psi}_k$. A state with this property is said to {\it regenerate} (or {\it renew}) infinitely often.
We can then conclude from Proposition 67 in \cite{serfozo2009basics}
that $T^{a_i}(k)$ is a regenerative process on $\psi_k$. Since every regenerative process is a renewal process,  $T^{a_i}(k)$ is a renewal process.
\end{proof}

\begin{algorithm}[h]
\SetAlgoLined
\DontPrintSemicolon
\SetNoFillComment
\textbf{Given:} $\mathcal{M}^{(a_i)}_{k}$, $\mathcal{M}^{(a_j)}_{k}$, $Y^{(a_i)}_k$,
$Y^{(a_j)}_k$\\
\For{$k \in1:t_{final}$}{
\For{$n_1 \in 1 : |\mathcal{N}|$}{
$l=1$; $n=1$;\\
\For{$n_2 \in n_1 + 1 : |\mathcal{N}|$}{
\If{$Y^{(n_1)}_k$ = $Y^{(n_2)}_k$}{
$R^{l}_{n}$=$\mathcal{M}^{(n_1)}_{k} \bigcup \mathcal{M}^{(n_2)}_{k}$;\\
$l=l+1$; $n=n+1$;\\
}
}
}
}
\caption{Renewal-reward computation for robots $(a_{i}, a_{j}) \in \mathcal{N}$}
\label{algo:RenewalReward}
\end{algorithm}

\begin{algorithm}[h]
\SetAlgoLined
\DontPrintSemicolon
\SetNoFillComment
\textbf{Given:} $X^{(a_i,a_j)}_{k-1}$,  $\hat{X}^{(a_i)}_k$, $\hat{X}^{(a_j)}_k$\\
\For{$l_1 \in 1 : size(|\hat{X}^{(a_i)}_k|,2)$}{
\For{$l_2 \in 1 : size(|\hat{X}^{(a_j)}_k|,2)$}{
$X_{temp}= \hat{X}^{(a_i)}_{k,l_1} \bigcup \hat{X}^{(a_j)}_{k,l_2}$; \\
\If{$X_{temp} \not \subset X^{(a_i,a_j)}_{k-1}$}{
$X^{(a_i,a_j)}_{k} = X^{(a_i,a_j)}_{k-1} \bigcup X_{temp}$;}
\Else{
$X^{(a_i,a_j)}_{k} = X^{(a_i,a_j)}_{k-1}$;}
}
$\mathcal{M}^{(a_i,a_j)}_k$ = $X^{(a_i,a_j)}_{k}$;
}
\caption{Exchange of set of estimated states between robots $(a_{i}, a_{j}) \in \mathcal{N}$}
\label{algo:UniqueSetTuplesConsensus}
\end{algorithm}

\section{Random Finite Sets Based Probability Hypothesis Density Filter}\label{sec:RFSPHD}
Let $M^{a_i}_{k} \leq n_m$ be the number of targets identified by quadrotor $a_i$ at time step $k$. Suppose that at time $k-1$, the target states are $ x_{k-1,1}^{a_i},~ x_{k-1,2}^{a_i},~ \ldots,~ x_{k-1,M^{a_i}_{k-1}}^{a_i} \in \mathcal{X}$, where $\mathcal{X}$ is the set of target states. At the next time step, some of these targets might disappear from the quadrotor's field of view (FoV), and new targets may appear. This results in $M^{a_i}_{k}$ new states $ x_{k,1}^{a_i}, ~x_{k,2}^{a_i}, ~ \ldots, ~x_{k,M^{a_i}_{k}}^{a_i}$. Note that the order in which the states are represented has no significance in the RFS multi-target tracking formulation. The quadrotor $a_i$ makes $N^{a_i}_{k}$ measurements $z_{k,1}^{a_i},~ \ldots,~ z_{k,N_{k}^{a_i}} \in \mathcal{Z}$ at time $k$, where $\mathcal{Z}$ is the set of measurements. The order in which the measurements are made is not significant. The states of the targets identified by quadrotor $a_i$ at time $k$ (i.e., the multi-target state) and the measurements obtained by the quadrotor at time $k$ can both be represented as finite sets: 
\begin{equation}
X_k^{a_i} = \{x_{k,1}^{a_i}, \ldots, x_{k,M^{a_i}_{k}}^{a_i} \}  \in \mathcal{F}(\mathcal{X}),
\label{eqn:RFStgtState}
\end{equation}
\begin{equation}
Z^{a_i}_{k} = \{ z_{k,1}^{a_i}, \ldots, z_{k,N_{k}^{a_i}} \} \in \mathcal{F}(\mathcal{Z}), 
\label{eqn:RFSMeasState}
\end{equation}
where $\mathcal{F}(\mathcal{X})$ is the multi-target state space and $\mathcal{F}(\mathcal{Z})$ is the measurement space. For a quadrotor $a_i$, given multi-target state $X_{k-1}^{a_i}$ at time $k-1$, each $x^{a_i}_{k-1} \in X_{k-1}^{a_i}$ either continues to exist (survives) at time $k$ with probability $p_{\mathbf{S},k}^{a_i}(x_{k-1}^{a_i})$ or disappears (dies) at time $k$ with probability $1 - p^{a_i}_{\mathbf{S},k}(x^{a_i}_{k-1})$. The conditional probability density at time $k$ of a transition from state $x^{a_i}_{k-1}$ to state $x^{a_i}_k$ is given by $f^{a_i}_{k|k-1}(\cdot|\cdot)$. 

We now define the RFS model for the time evolution of the multi-target state, which incorporates motion of the targets relative to the quadrotor, appearance (birth) of targets, and disappearance (death) of targets: 
\begin{equation}
X_k^{a_i} = \Bigg[ \bigcup_{\xi \in X_{k-1}^{a_i}} \mathbf{S}^{a_i}_{k|k-1}(\xi) \Bigg]  \bigcup \Bigg[ \bigcup_{\xi \in X_{k-1}^{a_i}} \mathbf{B}^{a_i}_{k|k-1}(\xi) \Bigg] \bigcup \mathbf{\Gamma}_{k}^{a_i}
\label{eqn:StateRFSComp}
\end{equation}

\begin{table}[h!]
    \centering
    \begin{tabular}{l l}
         $\mathbf{S}^{a_i}_{k|k-1}(\xi)$: & RFS of targets with previous state $\xi$ at time $k-1$ that survive at time $k$\\
         $\mathbf{B}^{a_i}_{k|k-1}(\xi)$: & RFS of targets spawned at time $k$ from targets with previous state $\xi$  at time $k-1$\\
         $\mathbf{\Gamma}^{a_i}_{k}$:& RFS of targets that are spontaneously born at time $k$
    \end{tabular}
    \label{tab:XComp}
\end{table}

At each time step, a quadrotor $a_i$ detects a target with state $x^{a_i}_k \in X^{a_i}_k$ with probability  $p_{\mathbf{D},k}^{a_i}(\cdot)$, or misses it with probability $1- p^{a_i}_{\mathbf{D},k}(\cdot)$. The conditional probability of obtaining a measurement $z^{a_i}_k \in Z^{a_i}_k$ from $x^{a_i}_k$ is characterized by the multi-target likelihood function, $g_{k}^{a_i}(\cdot|\cdot)$. 
We can now define the RFS model for the time evolution of the multi-target measurement, which incorporates  measurements of actual targets along with clutter:
\begin{equation}
    Z_k^{a_i} = \mathbf{K}_{k}^{a_i} \bigcup \Bigg[  \bigcup_{x \in X_k^{a_i}} \mathbf{\Theta}^{a_i}_{k}(x) \Bigg]
    \label{eqn:MeasRFSComp}
\end{equation}
\begin{table}[h!]
    \centering
    \begin{tabular}{l l}
         $\mathbf{K}^{a_i}_{k}$: & RFS of measurements arising from clutter at time $k$ \\
         $\mathbf{\Theta}^{a_i}_{k}(x)$:& RFS of measurements of the multi-target state $X_k^{a_i}$ at time $k$ 
   \end{tabular}
    \label{tab:ZComp}
\end{table}
 
The multi-target Bayes filter propagates the multi-target posterior density $p_k^{a_i}(\cdot~|~Z^{a_i}_{1:k})$ in time via recursion as: 
\begin{equation}
    p_{k|k-1}^{a_i}(X^{a_i}_k | Z_{1:k-1}^{a_i}) = \int_{X^{a_i} \in \mathcal{F}(\mathcal{X})} f^{a_i}_{k|k-1}(X_k^{a_i} |X^{a_i}) p_{k-1}^{a_i}(X^{a_i}|Z_{1:k-1}^{a_i}) \mu_s(d X^{a_i})
    \label{eqn:MultiTargetPredicted}
\end{equation}
\begin{equation}
    p_k^{a_i}(X^{a_i}_k |Z^{a_i}_{1:k}) = \frac{g_k^{a_i}(Z^{a_i}_{k} | X^{a_i}_{k}) p_{k|k-1}^{a_i}(X^{a_i}_k | Z_{1:k-1}^{a_i})} {\int_{X^{a_i} \in \mathcal{F}(\mathcal{X})}g_k(Z^{a_i}_{k} | X^{a_i}) p_{k|k-1}^{a_i}(X^{a_i} | Z_{1:k-1}^{a_i}) \mu_s(dX^{a_i})}
    \label{eqn:MulitTargetPosterior}
\end{equation}
where $\mu_S$ is a suitable reference measure on $\mathcal{F}(\mathcal{X})$ of target states $X^{a_i} \in \mathcal{{F}(\mathcal{X})}$, $g_k^{a_i}(\cdot|\cdot)$ represents the multi-target likelihood function, and $f_{k|k-1}^{a_i}(\cdot|\cdot)$ represents the multi-target transition density. For further details, see \cite{clark2006gm, vo2006gaussian}.

We will approximate the integrals above using the framework of the probability hypothesis density (PHD) filter, with the assumptions that: (1) each target evolves and generates observations independently of the others; (2) clutter is Poisson distributed and independent of target-originated measurements; (3) the multi-target RFS is Poisson distributed. For a RFS $X^{a_i} \in \mathcal{X}$ with probability distribution $p^{a_i}(\cdot)$, there is a non-negative function $v$ on $\mathcal{X}$, defined as the {\it intensity function}, such that for each region $\mathcal{S} \subset \mathcal{X}$,
\begin{equation}
    \int |X^{a_i} \cap \mathcal{S} |p^{a_i} (dX) = \int_{\mathcal{S}} v(x) dx. 
    \label{eqn:RFSIntensityFunction}
\end{equation} 
Then we can model the posterior intensity and its recursion as follows: 
\begin{equation}
    v^{a_i}_{k|k-1}(x) = \int p_{\mathbf{S},k}^{a_i}(\xi) f^{a_i}_{k|k-1}(x | \xi) v^{a_i}_{k-1}(\xi)d\xi + \int \beta^{a_i}_{k|k-1}(x|\xi) v^{a_i}_{k-1}(\xi) d\xi + \gamma^{a_i}_{k}(x),
    \label{eqn:Predicted_Intensity}
\end{equation}
\begin{equation}
    v_k^{a_i}(x) = [1-p^{a_i}_{\mathbf{D},k}(x)] v^{a_i}_{k|k-1}(x) + \sum_{z \in Z^{a_i}_k} \frac{p^{a_i}_{\mathbf{D},k}(x)g^{a_i}_k(z|x)v^{a_i}_{k|k-1}(x)}
    {\kappa^{a_i}_{k}(z) + \int p^{a_i}_{\mathbf{D},k}(\xi) g^{a_i}_{k}(z|\xi)v^{a_i}_{k|k-1}(\xi)}.
    \label{eqn:Updated_Intensity}
\end{equation}
In these equations, $v_k^{a_i}$ and $v_{k|k-1}^{a_i}$ denote the  intensities associated with, respectively, the multi-target posterior density $p_k^{a_i}(\cdot|\cdot)$ and the multi-target predicted density $p^{a_i}_{k|k-1}(\cdot|\cdot)$ that are defined by the recursion in \autoref{eqn:MultiTargetPredicted} and \autoref{eqn:MulitTargetPosterior}. The function $\gamma_{k}^{a_i}(\cdot)$ is the intensity of the RFS $\mathbf{\Gamma}_k^{a_i}$,  $\beta^{a_i}_{k|k-1}(\cdot|\xi)$ is the intensity of the RFS $\mathbf{B}_{k|k-1}(\xi)$,  and $\kappa_{k}^{a_i}(\cdot)$ is the intensity of the RFS $\mathbf{K}^{a_i}_k$. The quadrotor $a_i$ can estimate the number of targets as  
\begin{equation}
    \hat{N} = \int v(x) dx.
    \label{eqn:PHDTargetEstimate}
\end{equation}
The estimate $\hat{N}$ is used to update the number of elements of $\mathcal{M}^{a_i}_k$, and the intensity $v^{a_i}_{k}(x)$ computed from \autoref{eqn:Updated_Intensity} is used to update the states of the $\hat{N}$ targets. Then each element of $\mathcal{M}^{a_i}_k$ is represented as the following tuple:
\begin{equation}
    m^{a_i}_{k,l} = \langle ~l,~ v^{a_i}_k(x) ~\rangle,
    \label{eqn:Target_Set}
\end{equation}
where $l$ is a label for the tracked target, such as one of its properties, e.g. its color, shape, size or its position in the environment.

\subsection{Gaussian Mixture PHD Filter}\label{subsec:GMPHD Filter}
The PHD filter as  described in \eqref{eqn:Predicted_Intensity} and  \eqref{eqn:Updated_Intensity} does not admit a closed-form solution in general, and the numerical integration suffers from the curse of dimensionality. Thus, for implementation purposes, we consider a sub-optimal solution of the PHD filter that models approximates it as a mixture of Gaussians, as described 
in \cite{vo2006gaussian}. The Gaussian Mixture PHD (GM-PHD) filter provides a closed-form solution to the PHD filter under the following assumptions:
\begin{enumerate}
    \item[$\textbf{A.1}$] Each target generates observations independently of the others.
    \item[$\textbf{A.2}$] The clutter process is Poisson distributed and is independent of target-generated measurements.
   \item[$\textbf{A.3}$] Each target's state evolves according to a linear model with Gaussian process noise, and each quadrotor's sensor has a linear measurement model with Gaussian sensor noise, i.e. 
    \begin{equation}
        f^{a_i}_{k|k-1}(x | \xi) = \mathbb{N}(x; F_{k-1}, Q_{k-1}),
        \label{eqn:Target_Dynamics}
    \end{equation}
    \begin{equation}
        g^{a_i}_{k}(z|\xi) = \mathbb{N}(z; H_{k} x, R_k),
        \label{eqn:Measurement_Likelihood}
    \end{equation}
    where the notation $\mathbb{N}(\cdot~; \mu, \sigma)$ denotes a Gaussian density with mean $\mu$ and covariance $\sigma$, $F_{k-1}$ is the state transition matrix, $Q_{k-1}$ is the process noise covariance, $H_k$ is the observation or measurement matrix, and $R_k$ is the sensor noise covariance.
    \item[$\textbf{A.4}$] The detection probability is state-dependent and is modeled as
\begin{equation}
    p^{a_i}_{\mathbf{D},k}(x) = 
    \begin{cases}
    p_{\mathbf{D}} & || q^{a_i}_k - x || \in \mathcal{B}_{r}(q^{a_i}_k), \\
    0 & \text{otherwise},
    \end{cases}
    \label{eqn:Detection_Probability}
\end{equation}
 where $q^{a_i}_k$ denotes the grid coordinates of robot $a_i$ at time $k$ and $\mathcal{B}_{r}(q^{a_i}_k)$ represents the FoV of the sensor on robot $a_i$, which we model as a disc of radius $r$ centered at the robot location $q^{a_i}_k$. The survival probability is assumed to be constant:
\begin{equation}
    p^{a_i}_{\mathbf{S},k}(x) = p_{\mathbf{S},k}.
    \label{eqn:Survival_Probability}
\end{equation} 
\item[$\textbf{A.5}$] The birth and spawning intensities are modeled as Gaussian mixtures of the form
\begin{equation}
    \gamma^{a_i}_{k}(x) = \sum_{i=1}^{J_{\gamma,k}} w^{(i)}_{\gamma,k} \mathbb{N}(x;\mu_{\gamma,k}^{(i)}, P_{\gamma,k}^{(i)})
    \label{eqn:Birth_Intensity_GMM},
\end{equation}
\begin{equation}
    \beta^{a_i}_{k|k-1}(x|\xi) =\sum_{j=1}^{J_{\beta,k}} w^{(j)}_{\beta,k} \mathbb{N}(x;F_{\beta,k-1}^{(j)}\xi+d_{\beta,k-1}^{(j)}, Q_{\beta,k-1}^{(j)}), 
    \label{eqn:Spwanning_Intensity_GMM}
\end{equation}
where $J_{\gamma,k}$, $w^{(i)}_{\gamma,k}$, $\mu_{\gamma,k}^{(i)}$, and $P_{\gamma,k}^{(i)}$ are known parameters of the birth intensity, and 
$J_{\beta,k}$, $w^{(i)}_{\beta,k}$, $F_{\beta,k-1}^{(j)}$, $d_{\beta,k-1}^{(j)}$,  $Q_{\beta,k-1}^{(j)}$, and $P_{\beta,k-1}^{(j)}$ are known parameters of the spawn intensity of a target with state $\xi$ at time $k-1$. For more details on the parameters, please refer to \cite{vo2006gaussian}.
\end{enumerate}

\vspace{3mm}

Using the above assumptions, we can rewrite \autoref{eqn:Predicted_Intensity} and \autoref{eqn:Updated_Intensity} as follows. The intensity associated with the multi-target predicted density can be approximated as a Gaussian mixture:
\begin{equation}
    v^{a_i}_{k|k-1}(x) = v^{a_i}_{\mathbf{S},k|k-1}(x) + v^{a_i}_{\beta, k|k-1}(x) + \gamma^{a_i}_{k}(x),
    \label{eqn:Predicted_Intensity_GMM}
\end{equation}
where 
\begin{equation}
    v^{a_i}_{\mathbf{S},k|k-1}(x) = p_{\mathbf{S},k} \sum_{i=1}^{J_{k-1}} w^{(i)}_{k-1} \mathbb{N}(x;\mu_{\mathbf{S},k|k-1}^{(i)}, P_{\mathbf{S},k|k-1}^{(i)}),
    \label{eqn:survival_GM_PHD}
\end{equation}
\begin{equation}
    \mu_{\mathbf{S},k|k-1}^{(i)} = F_{k-1}\mu_{k-1}^{(i)},
    \label{eqn:mu_survival_GM_PHD}
\end{equation}
\begin{equation}
    P_{\mathbf{S},k|k-1}^{(i)} = Q_{k-1} + F_{k-1} P_{k-1}^{(i)} F_{k-1}^{T},
    \label{eqn:sigma_survival_GM_PHD}
\end{equation}
\begin{equation}
    v^{a_i}_{\beta, k|k-1}(x) = \sum_{i=1}^{J_{k-1}} \sum_{l=1}^{J_{\beta,k}} w^{(i)}_{k-1} w^{(l)}_{\beta,k} \mathbb{N}(x;\mu_{\beta,k|k-1}^{(i,l)},P_{\beta,k|k-1}^{(i,l)}),
    \label{eqn:birth_GM_PHD}
\end{equation}
\begin{equation}
    \mu_{\beta,k|k-1}^{(i,l)} = F_{\beta,k-1}^{(l)} \mu_{k-1}^{(i)} + d_{\beta,k-1}^{(l)},
    \label{eqn:mu_birth_GM_PHD}
\end{equation}
\begin{equation}
    P_{\beta,k|k-1}^{(i,l)} = Q_{\beta,k-1}^{(l)} + F_{\beta,k-1}^{(l)} P_{\beta,k-1}^{(i)} (F_{\beta,k-1}^{(l)})^{T}, 
    \label{eqn:sig_birth_GM_PHD}
\end{equation}
in which $J_{k-1}$, $w^{(i)}_{k-1}$, $\mu_{k-1}^{(i)}$, and $P_{k-1}^{(i)}$ are known parameters of the intensity function at time $k-1$ \cite{vo2006gaussian}.

Then the intensity associated with the multi-target posterior density can be approximated as a Gaussian mixture:
\begin{equation}
    v_k^{a_i}(x) = [1-p^{a_i}_{\mathbf{D},k}(x)] v^{a_i}_{k|k-1}(x) + \sum_{z \in Z^{a_i}_k}  v^{a_i}_{\mathbf{D},k}(x;z),
    \label{eqn:update_intensity_GM_PHD}
\end{equation}
where 
\begin{equation}
    v^{a_i}_{\mathbf{D},k}(x;z) = \sum_{j=1}^{J_{k|k-1}} w^{(j)}_k(z) \mathbb{N}(x;\mu^{(j)}_{k|k}(z),P_{k|k}^{(j)}),
    \label{eqn:detection_update_GM_PHD}
\end{equation}
\begin{equation}
    w^{(j)}_{k}(z) = \frac{p^{a_i}_{\mathbf{D},k}(x) w^{(j)}_{k|k-1}\mathbb{N}(z;H_{k}\mu_{k|k-1},R_{k} + H_{k} P_{k} H_{k}^{T})}{\kappa_{k}^{a_i}(z) + p^{a_i}_{\mathbf{D},k}(x) \sum_{l=1}^{J_{k|k-1}}w^{(l)}_{k|k-1} \mathbb{N}(z;H_{k}\mu_{k|k-1},R_{k} + H_{k} P_{k} H_{k}^{T})},
    \label{eqn:update_weight_GM_PHD}
\end{equation}
\begin{equation}
    \mu^{(j)}_{k|k}(z) = \mu^{(j)}_{k|k-1}(z) + K_{k}^{(j)}(z-H_{k}\mu^{(j)}_{k|k-1}(z)),
    \label{eqn:mu_update_GM_PHD}
\end{equation} 
\begin{equation}
    P_{k|k}^{(j)} = [I-K_{k}^{(j)}H_{k}]P^{(j)}_{k|k-1},
    \label{eqn:sig_update_GM_PHD}
\end{equation}
\begin{equation}
    K_{k}^{(j)} = P^{(j)}_{k|k-1}H_{k}^{T} [H_{k} P^{(j)}_{k|k-1}H_{k}^{T} + R_{k}]^{-1}.
    \label{eqn:Kalman_gain}
\end{equation}

\section{Simulation Results}\label{sec:SimResults}
In this section, we validate our approach with simulations in MATLAB. First, we model a scenario with a bounded environment with dimensions 5m $\times$ 5m that contains 3 stationary targets which must be located by 3 robots, as shown in \autoref{fig:Environment_with_Targets_4}. The state of each target is defined as its $x-y$ position coordinates, $x = [p_x, p_y]^{T}$. 
A robot's sensor measurement of a target's state is modeled according to Equation \eqref{eqn:Measurement_Likelihood}.
Each robot has a circular FoV of radius $r_{FOV} = 0.6$m, centered at the robot's position on the spatial grid. We assume that each robot is able to accurately localize itself on the grid, and that there are no obstacles present in the environment.

Since each agent has a limited FoV, We assume that the targets that are detected at time step $k$ survive in the next time step with probability $p_{\mathbf{S},k} = 0.1$ for all robots. Since the targets are stationary, $F_{k-1} = \mathbf{I}_{2}$, the $2 \times 2$ identity matrix. We also set $Q_{k-1} = 0.2\mathbf{I}_{2}$. As the robots explore the environment, new targets might appear in their FoV. We account for this by allowing 4 new targets 
to be birthed at each time step, depending upon the robot's position on the grid, with weights $w_{\gamma,k} =(w_{\gamma,k}^{(i)})_{i=1}^{4} = [0.1, 0.1, 0.1, 0.1]^{T}$. Thus, the birth intensity at each time step from Equation \eqref{eqn:Birth_Intensity_GMM} is modeled as
\begin{equation}
    \gamma^{a_i}_{k}(x) = 0.1 \mathbb{N}(x;\mu_{\gamma,k}^{(1)}, P_{\gamma,k}^{(1)}) +
    0.1 \mathbb{N}(x;\mu_{\gamma,k}^{(2)}, P_{\gamma,k}^{(2)}) +
    0.1 \mathbb{N}(x;\mu_{\gamma,k}^{(3)}, P_{\gamma,k}^{(3)}) + 
    0.1 \mathbb{N}(x;\mu_{\gamma,k}^{(4)}, P_{\gamma,k}^{(4)}),
\end{equation}
where $P_{\gamma}^{(l)} = 0.5\mathbf{I}_{2}$ and
\begin{equation}
    \mu_{\gamma}^{(l)} = 
    \begin{bmatrix}
    p_{x,k}^{a_i} +r_{birth}\cos(\theta_{l}) \\
    p_{y,k}^{a_i} +r_{birth}\sin(\theta_{l})
    \end{bmatrix},
\end{equation} 
in which $q^{a_i}_{k} =[p_{x,k}^{a_i} p_{y,k}^{a_i}]$ denotes the $x-y$ coordinates of robot $a_i$ at time step $k$, corresponding to its current node $Y^{a_i}_k$;  $r_{birth}=0.8 r_{FOV}$, so that the targets are birthed only near the boundary of FOV; and
$\theta_{l} = [\pi/4,~  3\pi/4, ~5\pi/4, ~7\pi/4] ^{T}$ rad, the angles at which targets are likely to appear. We assume that there are no spawned targets.  Each target is detected with a probability of $p_{\mathbf{D}} = 0.8$, and a quadrotor's observation of a target follows the measurement model \eqref{eqn:Measurement_Likelihood} with $H_k = \mathbf{I}_2$ and $R_k = 0.25\mathbf{I}_2$. The observations are immersed in clutter that can be modeled as a Poisson RFS $\mathbf{K}_k^{(\cdot)}$ with intensity  $\kappa_{k}^{(\cdot)}(z) = \lambda_C A_{s} \mathbb{U}(z)$, where $\lambda_C = 3.98 \times 10^{-3}$ is the clutter intensity; $A_{s}$ is the area of the sensor's circular FoV, which is approximately 1m$^2$; and $\mathbb{U}(z)$ is the uniform density over $A_s$.

We assume that all robots start at random positions on the grid and have no knowledge of the number of targets or their states (positions). The robots explore the environment according to the random walk model  \eqref{eqn:DTDSMarkovChain}. As the robots detect the targets, they recursively update their estimates of the number of targets and their positions using the GM-PHD framework described in Section \ref{subsec:GMPHD Filter}. We set $T=1\times 10^{-3}$ as the pruning threshold and $U=4$ as the merging distance threshold (see Table II in \cite{vo2006gaussian} for details on these parameters). 
\begin{figure}[t]
    \centering
    \includegraphics[height=0.4\textwidth]{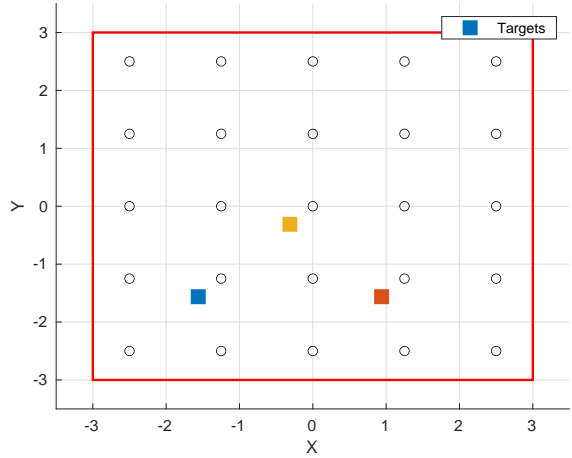}
    \caption{A 5m $\times$ 5m square environment, with hollow circles denoting the grid nodes and squares denoting the targets. The red border denotes the boundary of the area that is explored by 3 robots.}
    \label{fig:Environment_with_Targets_4}
\end{figure}
\begin{figure}[h]
\begin{subfigure}{0.48\textwidth}
    \centering
    \includegraphics[height=0.7\textwidth,width=\textwidth]{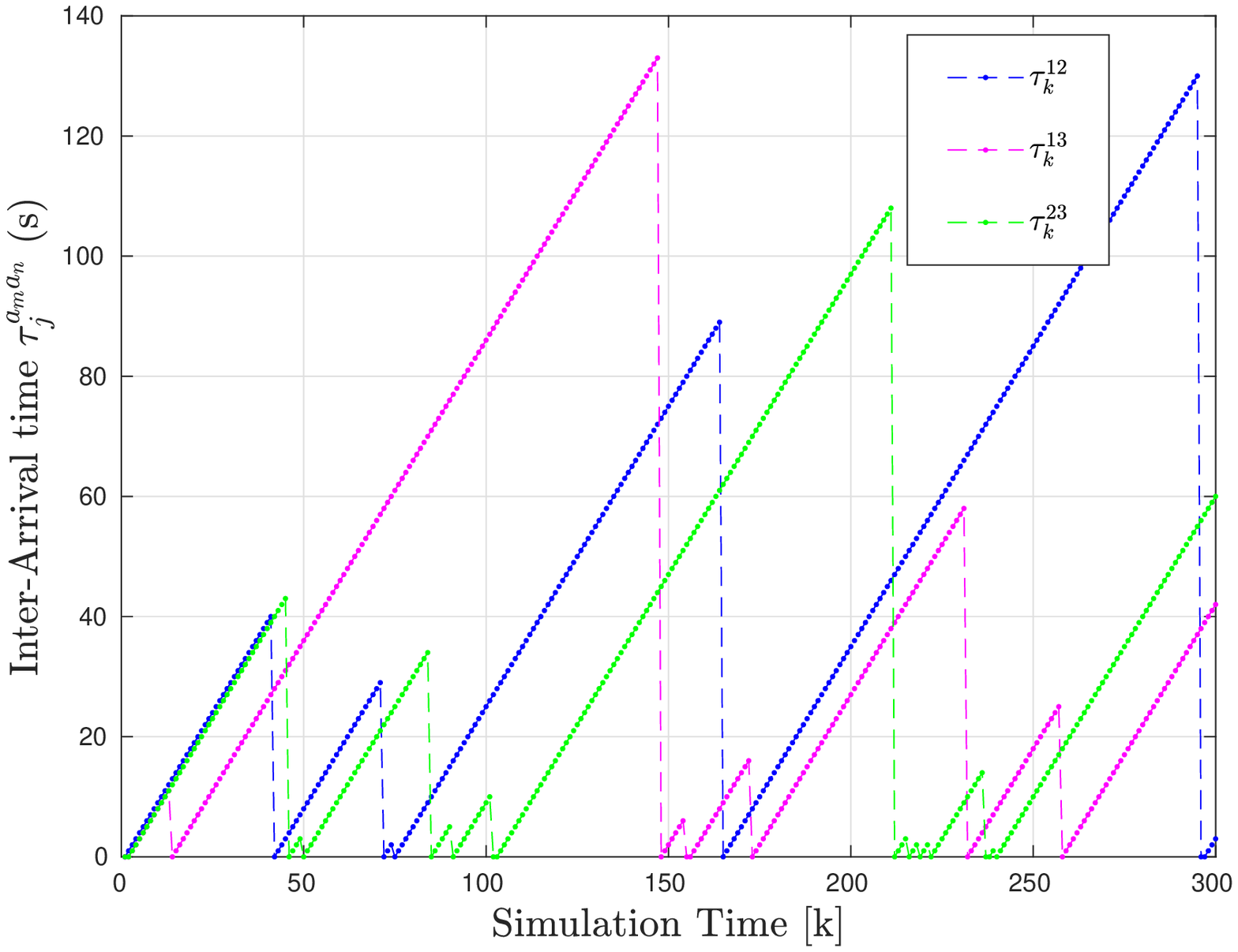}
    \caption{}
    \label{fig:Renewal_process_3_agents_5mx5m_grid}
    \end{subfigure}
    \begin{subfigure}{0.48\textwidth}
    \centering
    \includegraphics[height=0.7\textwidth,width=\textwidth]{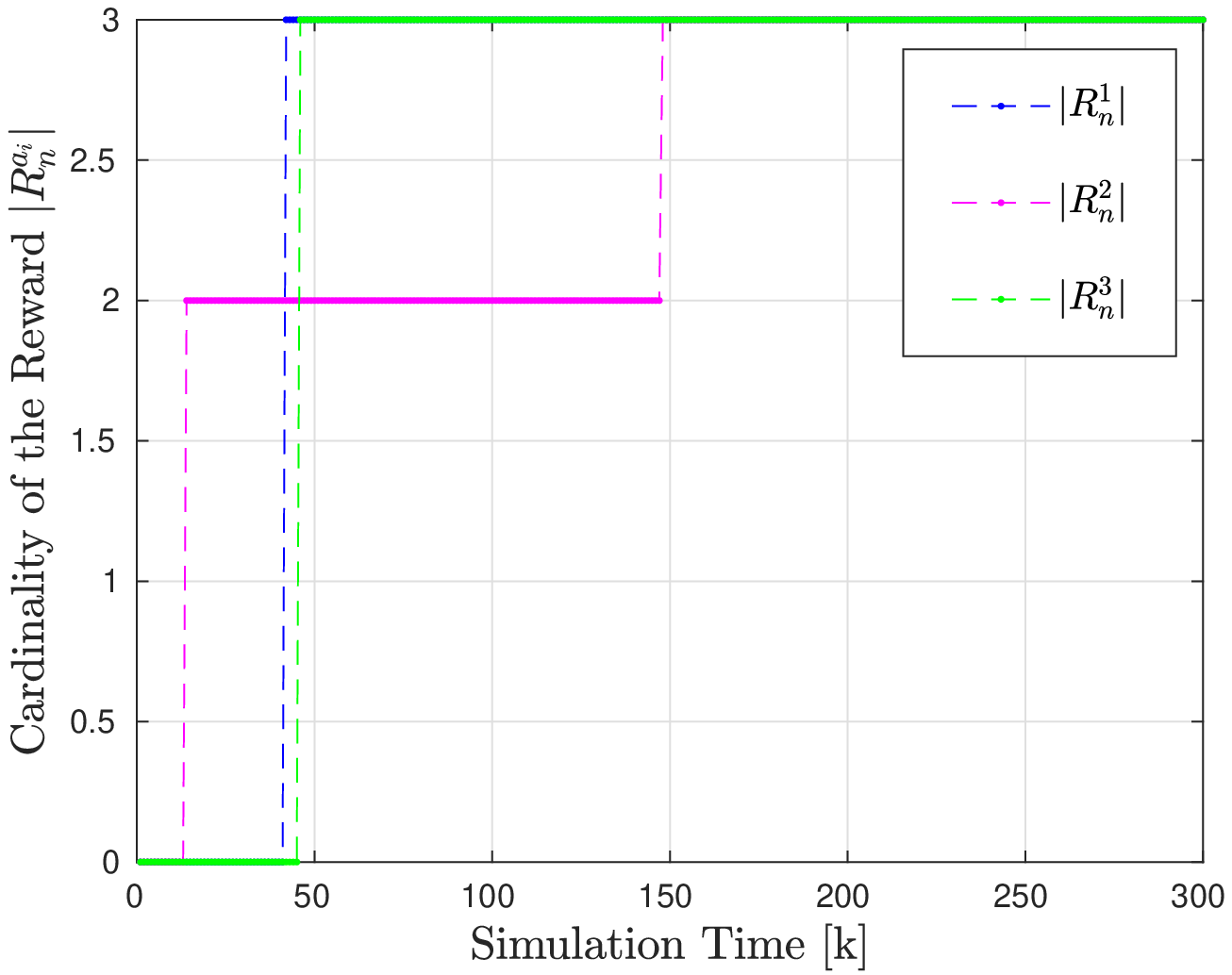}
    \caption{} 
    \label{fig:Reward_process_3_agents_5mx5m_grid}
    \end{subfigure}
    \caption{(a) Inter-arrival times during a simulation of 3 robots exploring the environment in \autoref{fig:Environment_with_Targets_4}. Renewal epochs, i.e. times when two robots meet at a node, are the times at the peaks of the graphs. Each renewal epoch marks the initialization of a new inter-arrival time.
    (b) Cardinality of the reward accumulated by each of the  robots at each time step.}
\end{figure}
\autoref{fig:Renewal_process_3_agents_5mx5m_grid} plots the inter-arrival times over time during 300 s of the simulation. Each inter-arrival time $\tau_j^{a_m a_n}$ ends at a renewal epoch, i.e. a time when any two robots $a_m$ and $a_n$ meet at a node, which can be identified in the figure as the time at the corresponding peak of the graph. At this time, the next inter-arrival time $\tau_{j+1}^{a_m a_n}$ is initialized to zero. 
\autoref{fig:Reward_process_3_agents_5mx5m_grid} plots the time evolution of the cardinality of the reward \eqref{eqn:MeetingReward} earned by each robot, which is the estimated number of targets. The average inter-arrival time over this simulation run was calculated to be $\mathbb{E}[\tau^{(\cdot)}_k] \sim 68$ s, and the time required for the cardinality of all robots' rewards to equal the actual number of targets, $n=3$, was $t_{reward} \sim 150$ s. 
Thus, for a scenario with both a robot density (number of robots per m$^2$) and a target density (number of targets per m$^2$) of $3/25 = 0.12$ m$^{-2}$, there must be about $\frac{t_{reward}}{\mathbb{E}[\tau^{(\cdot)}_j]} \approx 2.2$ renewals, i.e. at least 2 renewals, for all robots to achieve the same reward cardinality (estimated number of targets).
\autoref{fig:Target_Posn_Detected_Agent1}, \autoref{fig:Target_Posn_Detected_Agent2}, and \autoref{fig:Target_Posn_Detected_Agent3} show the true positions of the targets and their estimated positions by each robot at the end of the simulation time.  \autoref{fig:GM_PHD_Inetnsity_agent_1}, \autoref{fig:GM_PHD_Inetnsity_agent_2}, and \autoref{fig:GM_PHD_Inetnsity_agent_3} show the corresponding PHD intensity for each robot as a Gaussian mixture model with $n=3$ components (the number of targets), computed from  Equation  \eqref{eqn:update_intensity_GM_PHD}.
We obtain the number of targets estimated by each robot $a_i$ as 
\begin{equation}
    \hat{N}^{a_i} = \ceil[\bigg]{\sum_{l=1}^{n}w_k^{(l),a_{i}}},
\end{equation} 
where the weights $w^{(\cdot), a_{i}}_{k}$ for robots $a_i = 1,2,3$ are given by the peak intensities in \autoref{fig:GM_PHD_Inetnsity_agent_1},  \autoref{fig:GM_PHD_Inetnsity_agent_2}, and \autoref{fig:GM_PHD_Inetnsity_agent_3}, respectively. The estimated positions of the targets are obtained from positions of these peak intensities.

We also evaluated our approach through Monte Carlo simulations of three scenarios, with 100 simulation runs for each scenario. In all scenarios, 20 robots explored a grid according to the random walk model \eqref{eqn:DTDSMarkovChain} in order to track a set of stationary targets. The robots were initialized at random positions on the grid, and the positions of the targets were kept the same over all 100 runs for each scenario. In Scenario 1, simulated for 1000 s, the grid has dimensions 15m $\times$ 15m and there are 10 targets; in Scenario 2, simulated for 2000 s, the grid has dimensions 20m $\times$ 20m and there are 15 targets; and in Scenario 3, simulated for 3000 s, the grid has dimensions 30m $\times$ 30m and there are 20 targets. The mean inter-arrival time and mean reward percentage  for each scenario, averaged over all 100 runs, are given in \autoref{tab:MonteCarloSims}. The  mean reward percentage is computed from the ratio of the mean number of targets detected by the robots until the mean inter-arrival time to the actual number of targets in the scenario. 
\autoref{tab:MonteCarloSims} shows that the mean inter-arrival time increases as the density of robots in the environment decreases, which is due to the lower rate of robots encounters with one another in larger environments, on average. The table also shows that as the density of targets in the environment decreases, the percentage of targets identified before the mean inter-arrival time increases, on average. This indicates that in the scenarios simulated, the longer inter-arrival times for larger environments tend to enable identification of a higher number of targets, despite the lower target density.
\begin{figure}[H]
\begin{subfigure}{0.45\textwidth}
\centering
    \includegraphics[height=0.6\textwidth,width=\textwidth]{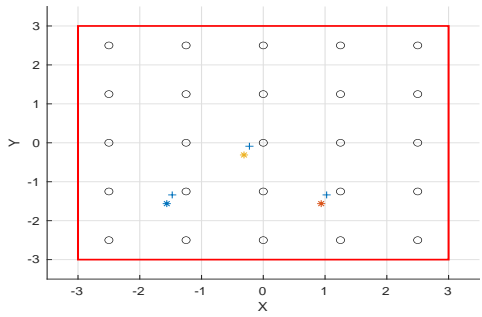}
    \caption{}
    \label{fig:Target_Posn_Detected_Agent1}
\end{subfigure}
\begin{subfigure}{0.45\textwidth}
    \centering
    \includegraphics[height=0.6\textwidth,width=\textwidth]{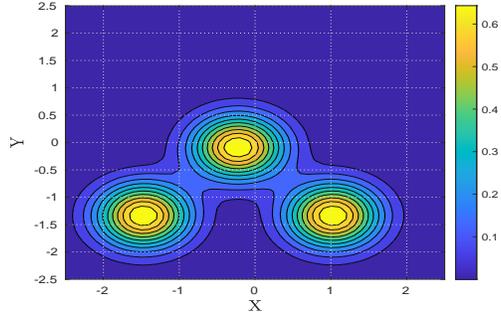}
    \caption{}
    \label{fig:GM_PHD_Inetnsity_agent_1}
    \end{subfigure}
    \caption{Multi-target tracking by robot 1. (a) Estimated ($*$) and true ($+$) target positions. (b) GM-PHD intensities computed from Equation \eqref{eqn:update_intensity_GM_PHD}.}
\end{figure}
\begin{figure}[H]
\begin{subfigure}{0.45\textwidth}
\centering
    \includegraphics[height=0.6\textwidth,width=\textwidth]{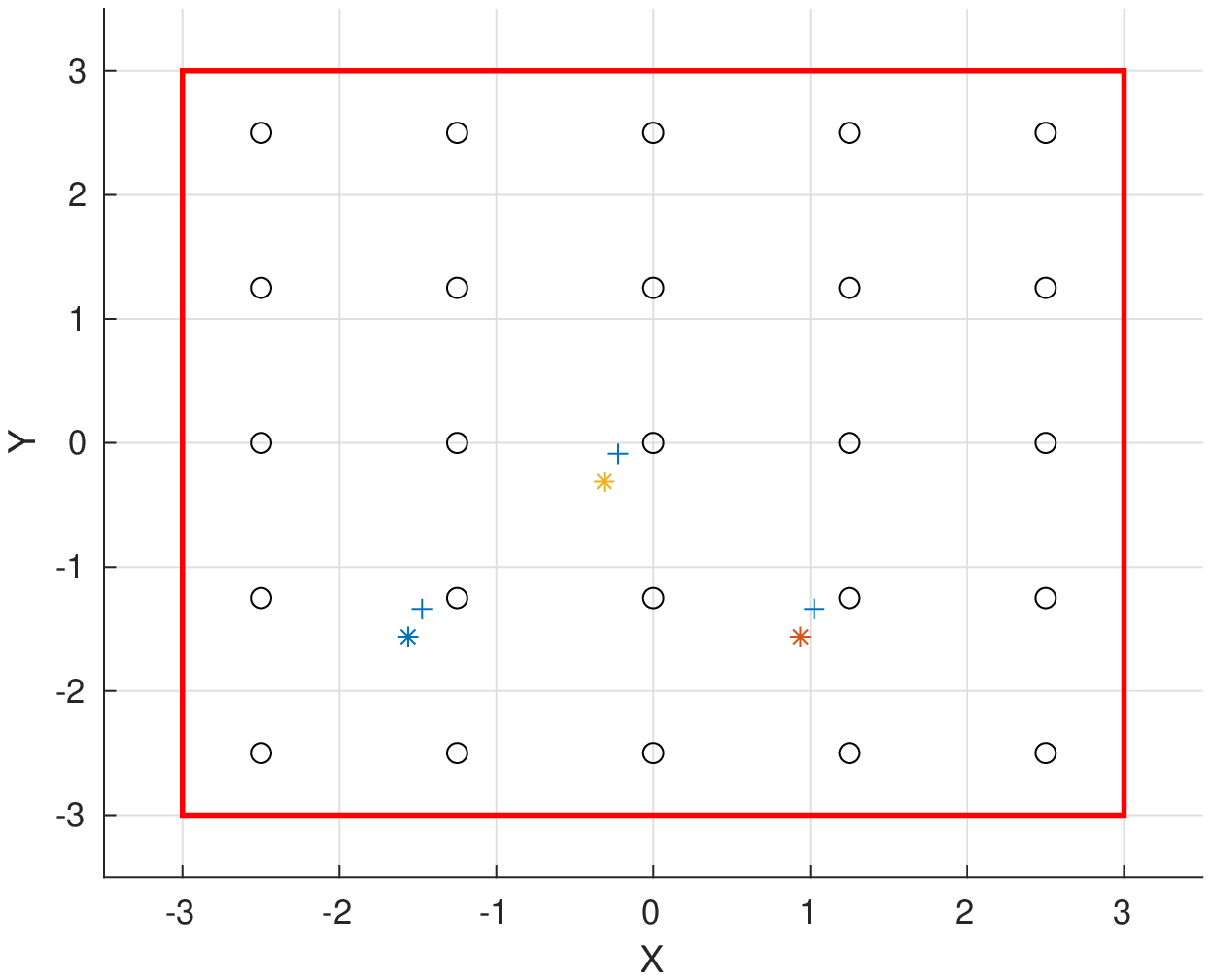}
    \caption{}
    \label{fig:Target_Posn_Detected_Agent2}
\end{subfigure}
\begin{subfigure}{0.45\textwidth}
    \centering
    \includegraphics[height=0.6\textwidth,width=\textwidth]{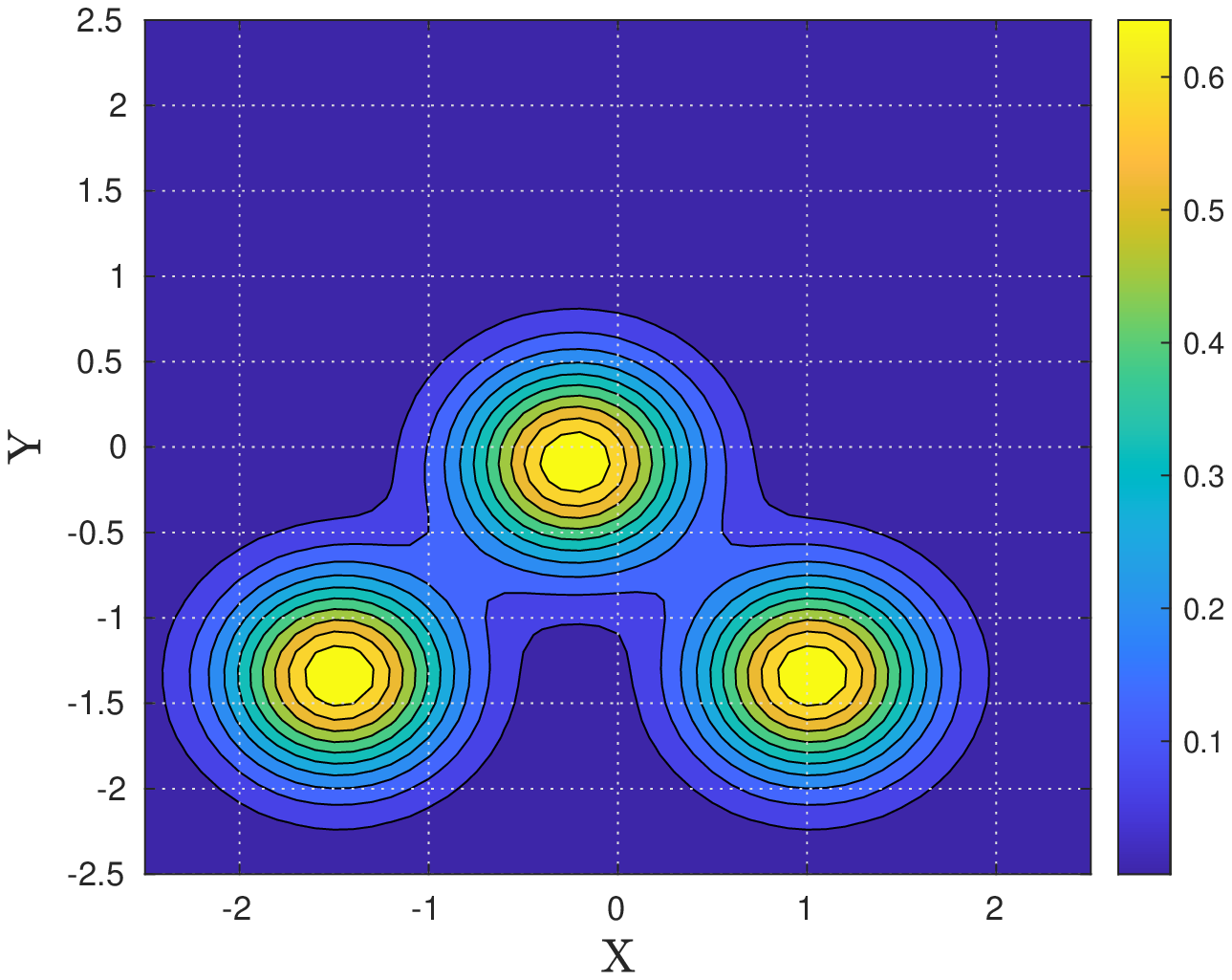}
    \caption{}
    \label{fig:GM_PHD_Inetnsity_agent_2}
\end{subfigure}
\caption{Multi-target tracking by robot 2. (a) Estimated ($*$) and true ($+$) target positions. (b) GM-PHD intensities computed from Equation \eqref{eqn:update_intensity_GM_PHD}.}
\end{figure}
\begin{figure}[H]
\begin{subfigure}{0.45\textwidth}
\centering
    \includegraphics[height=0.6\textwidth,width=\textwidth]{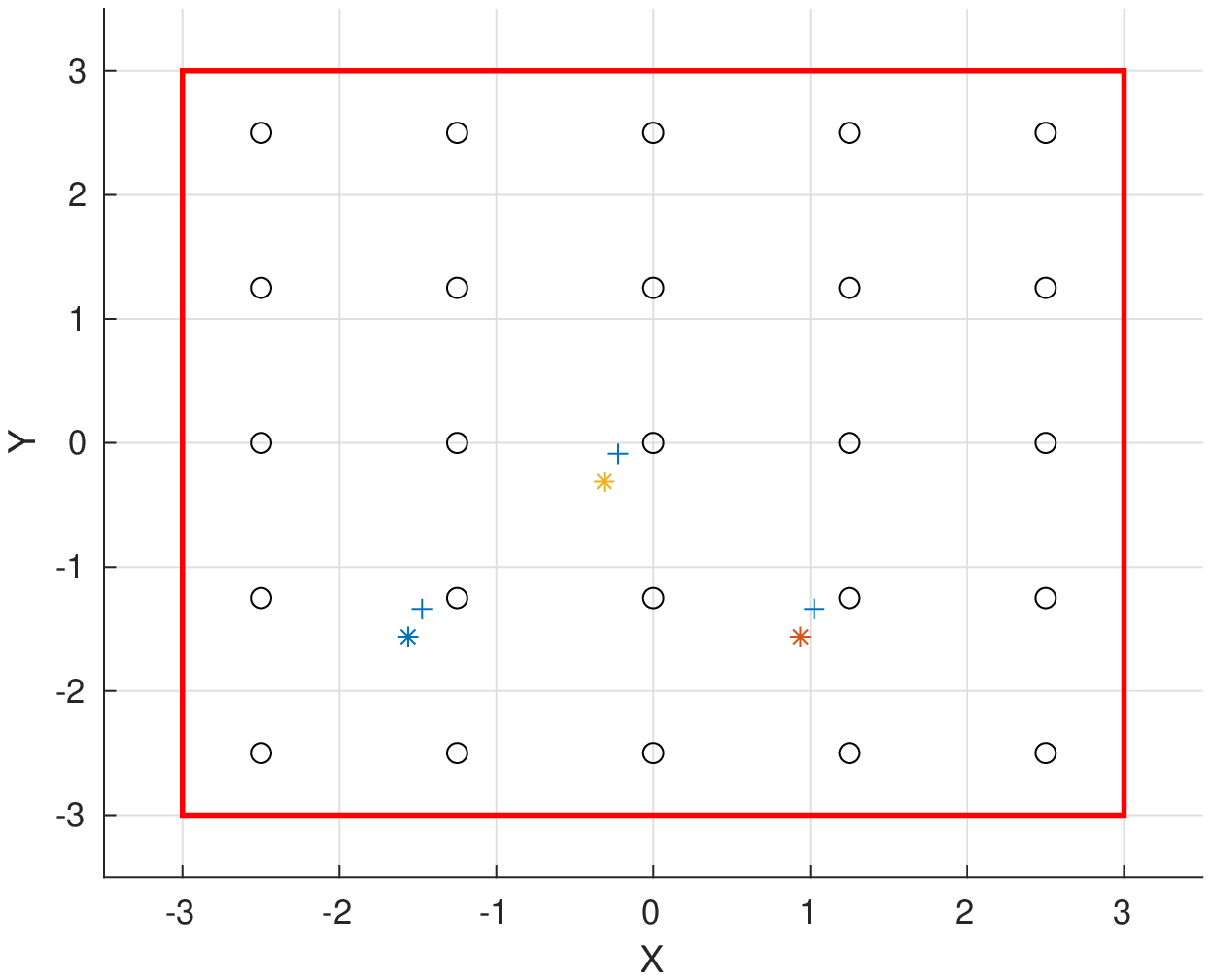}
    \caption{}
    \label{fig:Target_Posn_Detected_Agent3}
\end{subfigure}
\begin{subfigure}{0.45\textwidth}
    \centering
    \includegraphics[height=0.6\textwidth,width=\textwidth]{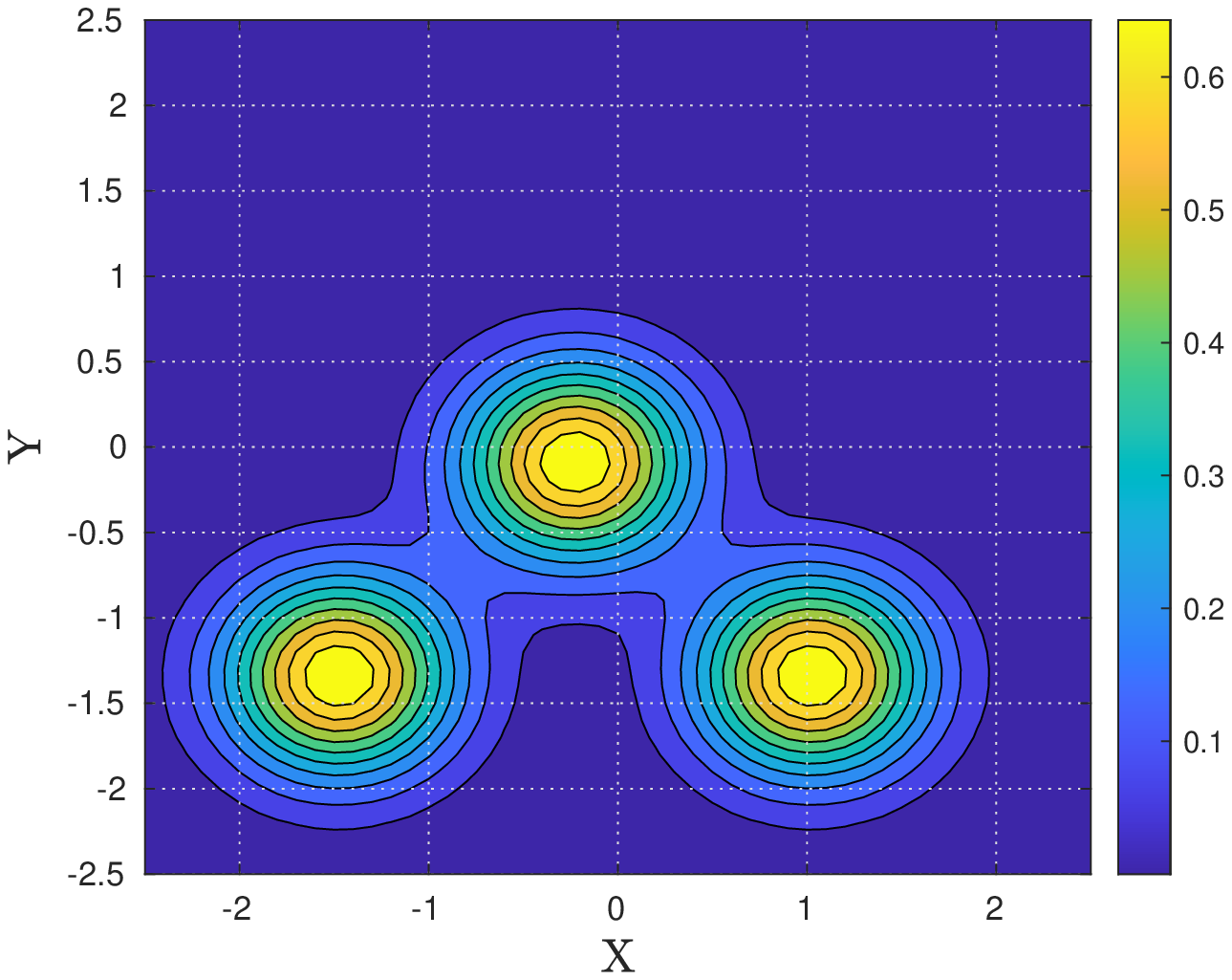}
    \caption{}
    \label{fig:GM_PHD_Inetnsity_agent_3}
\end{subfigure}
\caption{Multi-target tracking by robot 3. (a) Estimated ($*$) and true ($+$) target positions. (b) GM-PHD intensities computed from Equation \eqref{eqn:update_intensity_GM_PHD}.}
\end{figure}

\begin{table}[H]
    \centering
     \begin{tabular}{|c |c |c |c |}
     \hline
     Scenarios & 1 & 2 & 3 \\
     \hline
     Mean inter-arrival time (s) & 20 & 190 & 430 \\
     Mean reward ($\%$) & 10 & 33 & 65 \\
     \hline
     \end{tabular}
     \caption{Mean inter-arrival time and mean reward percentage over 100 simulation runs each for 3 scenarios.} 
     \label{tab:MonteCarloSims}
\end{table}
\section{Conclusion and Future Work}\label{sec:Conc}

In this paper, we demonstrated theoretically that a group of robots equipped with limited sensing and communication capabilities, moving according to a DTDS Markov chain model on a spatial grid, is able to detect and track the number and states of multiple stationary targets in the environment using the Gaussian Mixture formulation of the PHD filter from the RFS framework. We verified our results with numerical simulations in MATLAB. 
In the future, we plan to implement this strategy on quadrotors equipped with RGBD cameras and 5G WiFi dongles for exchanging data between the robots.

\section{Acknowledgments}
This work was supported by the Arizona State University Global Security Initiative. The authors utilized {HPC} resources provided by Research Computing at Arizona State University to generate results reported in this paper.

\bibliography{main}

\end{document}